\documentclass{article}
\usepackage{spconf,amsmath,amssymb,amsfonts,graphicx,bm,amsthm,url}
\usepackage{color}
\usepackage{mathtools}
\usepackage{adjustbox}
\usepackage{hyperref}

\DeclareMathOperator*{\argmax}{arg\,max}

\title{On Extensions of CLEVER: a Neural Network Robustness Evaluation Algorithm}
\name{Tsui-Wei Weng\textsuperscript{1,3*}, Huan Zhang\textsuperscript{2}\sthanks{Equally contributed. Codes: \url{https://github.com/huanzhang12/CLEVER}.}, Pin-Yu Chen\textsuperscript{3}, Aurelie Lozano\textsuperscript{3}, Cho-Jui Hsieh\textsuperscript{2}, Luca Daniel\textsuperscript{1}}

\address{\textsuperscript{1}Massachusetts Institute of Technology, Cambridge, MA 02139 \\
\textsuperscript{2}University of California, Los Angeles, CA 90095 \\
\textsuperscript{3}IBM Research, Yorktown Heights, NY 10598}

\begin{document}
%
\maketitle
\begin{abstract}
CLEVER (Cross-Lipschitz Extreme Value for nEtwork Robustness) is an Extreme Value Theory (EVT) based robustness score for large-scale deep neural networks (DNNs). In this paper, we propose two extensions on this robustness score. First, we provide a new formal robustness guarantee for classifier functions that are twice differentiable. We apply extreme value theory on the new formal robustness guarantee and the estimated robustness is called second-order CLEVER score. Second, we discuss how to handle gradient masking, a common defensive technique, using CLEVER with Backward Pass Differentiable Approximation (BPDA). With BPDA applied, CLEVER can evaluate the \textit{intrinsic} robustness of neural networks of a broader class -- networks with non-differentiable input transformations. We demonstrate the effectiveness of CLEVER with BPDA in experiments on a 121-layer Densenet model trained on the ImageNet dataset.
\end{abstract}
\section{Introduction}
\label{sec:intro}

It is well-known that deep neural networks (DNNs) are vulnerable to adversarial examples, and a small perturbation added to the input can mislead the network to classify in any desired class. There has been significant efforts developing verification techniques to prove that no adversarial perturbation $\delta$ exists if $\| \delta \|_p \leq r$ given an input $\bm{x_0}$ and a classifier function $f$. However, the verification problem is hard and generally intractable because a general neural network classifier is highly non-convex and non-smooth. 

Alternatively, instead of verifying the exact robustness $r$, one idea is to provide a \textit{lower bound} of $r$, which guarantees that no adversarial examples exist within an $\ell_p$ ball of radius $\epsilon$. We call $\epsilon$ the \textit{robustness lower bound} of the input image $\bm{x_0}$ on classifier function $f$. CLEVER (Cross-Lipschitz Extreme Value for nEtwork Robustness)~\cite{weng2018evaluating} is the first attack-agnostic robustness score to estimate the robustness lower bound $\epsilon$ for large-scale DNNs, e.g. modern ImageNet networks such as ResNet, Inception, etc. It is based on a theoretical analysis of formal robustness guarantee with Lipschitz continuity assumption. The authors of \cite{weng2018evaluating} propose a sampling based approach with Extreme Value Theory to estimate the local Lipschitz constant, and empirically, this estimation aligns well with other robustness evaluation metrics, for example, the distortion of adversarial perturbation found by strong attacks.

In this work, we provide two extensions of CLEVER. First, we derive a new robustness guarantee for classifier functions that are twice differentiable, and we estimate the theoretical bounds via extreme value theory. Second, we extend CLEVER to be capable of evaluating the robustness of networks with non-differentiable input transformations, making it available for a wider class of neural networks deployed with gradient masking based defense.

\section{Related Work}
\label{sec:related}

Evaluating the robustness of a neural network can be done by crafting adversarial examples with a specific attack algorithm~\cite{carlini2017towards,bastani2016measuring,chen2017ead,moosavi2016deepfool}. 
However, this methodology has a major drawback as the resilience of a network to existing attacks is not guaranteed to be extended to subsequent attacks.
In fact, many defensive methods have been shown either partially or completely broken after stronger and adaptive attacks are proposed~\cite{athalye2018obfuscated,athalye2018robustness,carlini2017magnet,carlini2017adversarial}. Thus, it is of great importance to provide an attack-agnostic robustness evaluation metric.

On the other hand, existing formal verification methods that solves the exact minimum adversarial distortion $r$ (which is independent of attack algorithm) are quite expensive -- verifying a small network with only a few hundred neurons on one input example can take a few hours~\cite{katz2017reluplex}, and in fact, even finding a non-trivial lower bound for $r$ can be hard, and so far only results on CIFAR and MNIST networks are available~\cite{weng2018towards,hein2017formal}.
\cite{weng2018evaluating} presents a framework to estimate local Lipschitz constant using extreme value theory, and then obtain an attack-agnostic robustness score (CLEVER) based on first-order Lipschitz continuity condition. CLEVER can scale to ImageNet networks.

Recently, Goodfellow~\cite{goodfellow2018gradient} raises concerns on CLEVER in the case of networks with gradient masking, a defensive technique that obfuscates model gradients to prevent gradient based attacks. One of the main objective of this work is to show that such concerns can be safely eliminated with the BPDA technique proposed in~\cite{athalye2018obfuscated}. Moreover, we also experimentally show how CLEVER can successfully handle networks with non-differentiable input transformations, including the stair-case function example in \cite{goodfellow2018gradient}.

\section{Extending CLEVER with Second Order Approximation}
\label{sec:theory}
\newtheorem{theorem}{Theorem}[section]
\newtheorem{lemma}[theorem]{Lemma}
\newtheorem{definition}[theorem]{Definition}
\newtheorem{notation}[theorem]{Notation}
\newtheorem{proposition}[theorem]{Proposition}
\newtheorem{corollary}[theorem]{Corollary}
\newtheorem{conjecture}[theorem]{Conjecture}
\newtheorem{assumption}[theorem]{Assumption}
\newtheorem{observation}[theorem]{Observation}
\newtheorem{fact}[theorem]{Fact}
\newtheorem{remark}[theorem]{Remark}
\newtheorem{claim}[theorem]{Claim}
\newtheorem{example}[theorem]{Example}
\newtheorem{problem}[theorem]{Problem}
\newtheorem{open}[theorem]{Open Problem}
\newtheorem{property}[theorem]{Property}
\newtheorem{hypothesis}[theorem]{Hypothesis}

\newcommand{\xo}{\bm{x_0}}
\newcommand{\x}{\bm{x}}
\newcommand{\xa}{\bm{x_a}}
\newcommand{\del}{\bm{\delta}}
\newcommand{\R}{\mathbb{R}}
\newcommand{\z}{\bm{z}}
\newcommand{\y}{\bm{y}}
\newcommand{\eps}{\epsilon}
\newcommand{\Lipsloc}{L_{q,x_0}^j}
\newcommand{\Ball}{\mathbb{B}_p(\xo,R)}
\newcommand{\Balltwo}{\mathbb{B}_2(\xo,R)}
\newcommand{\Ballx}{\mathbb{B}_p(\x,R)}

\subsection{Background and definitions}
Let $\xo \in \R^d$ be the input of a $K$-class classifier $f: \R^d \rightarrow \R^K$, the predicted class of $\xo$ is $c(\xo) = \argmax_{1\leq i \leq K} f_i(\xo)$. Given $\xo$ and $c$, we say $\xa := \xo + \del$ is an adversarial example if there exists a $\del \in \R^d$ makes $c(\xa) \neq c(\xo)$ while $\| \del \|_p$ is small. A successful \textit{untargeted attack} is to find a $\bm{x_a}$ such that $c(\bm{x_a}) \neq c(\bm{x_0})$ while a successful \textit{targeted attack} is to find a $\bm{x_a}$ such that $c(\bm{x_a})=t$ given a target class $t \neq c(\bm{x_0})$. On the other hand, the definition of norm-bounded robustness $\eps$ is the following: given a target class $t$,  $\eps$ is the \textit{targeted} robustness of $\xo$, if 
\begin{equation}
    \label{eq:target_robustness}
    g_t(\xo+\del) \geq 0, \, \forall~\| \del \|_p \leq \eps, 
\end{equation}
where $g_t(\x) := f_c(\x)-f_t(\x)$.  Similarly, $\eps$ is the \textit{untargeted} robustness if \eqref{eq:target_robustness} holds for all classes $t \neq c(\xo)$. 

\subsection{Robustness for continuously differentiable classifiers}
In \cite{weng2018evaluating}, the authors have shown that if the classifier function $f$ has continuously differentiable components $f_i$, the targeted robustness is 
\begin{equation}
\label{eq:1st_order_bnd}
    \eps = \min(\frac{g_t(\bm{x_0})}{L_q^t},R),
\end{equation}
where $L_q^t$ is the local Lipschitz constant for the function $g_t(\bm{x})$ within a local region $\x \in \Ball$ and $1/p+1/q = 1, \, 1, \leq p,q \leq \infty$.
A simple proof of this guarantee is based on the mean value theorem on the first order expansion of $g_t(\xo+\del)$: 
\begin{equation}
    \label{eq:g_x_first}
    \exists s \in [0,1], \enskip g_t(\xo+\del) = g_t(\xo) + \nabla g_t(\xo+s\del)^\top \del.
\end{equation}
With H\"older's inequality,  
\begin{align*}
    \label{eq:g_x_first_proof}
    g_t(\xo+\del) &= g_t(\xo) + \nabla g_t(\xo+s\del)^\top \del \\
    & \geq g_t(\xo) - \|\nabla g_t(\xo+s\del)\|_q \|\del\|_p \\
    & \geq g_t(\xo) - \max_{\x \in \Ball} \| \nabla g_t(\x) \|_q \cdot \| \del \|_p \\
    & = g_t(\xo) - L_q^t \cdot \|\del\|_p.
\end{align*}
Thus, the targeted robustness bound \eqref{eq:1st_order_bnd} is obtained by requiring the lower bound of $g_t(\xo+\del)$ to be non-negative. The authors of~\cite{weng2018evaluating} further extend their analysis to neural networks with ReLU activations, which is a special case of \textit{non-differentiable} functions.

\subsection{Robustness for twice differentiable classifiers}
In this work, we provide formal robustness guarantees when classifier functions $f$ are twice differentiable -- for example, neural networks with twice differentiable activations such as tanh, sigmoid, softplus, etc. For a twice-differentiable function $g_t(\x) := f_c(\x)-f_t(\x)$, there exists $s \in [0,1]$ such that 
\begin{equation}
    \label{eq:g_x}
    g_t(\xo+\del) = g_t(\xo) + \nabla g_t(\xo)^\top \del + \frac{1}{2} \del^\top \bm{H}(\xo+s\del) \del,
\end{equation}
where $\bm{H}(\xo+s\del)$ is the Hessian of $g_t$ at $\xo+s\del$. This is analogous to the Mean Value Theorem in the first order case, but extended with a second order term. This expansion of $g_t(\xo+\del)$ can be used to derive the targeted robustness of $\xo$ in the following Theorem: 
\begin{theorem}[Formal robustness guarantee] 
\label{thm:delta_bnd}
Given an input $\xo$ and a $K$-class classifier $f$, the targeted robustness of $\xo$ is 
\begin{equation}
\label{eq:2nd_order_bnd}
	\eps = \min(\frac{-b+\sqrt{b^2+2a\gamma}}{a}, R)
\end{equation}
where $a = \max_{\x \in \Ball} \| \bm{H}(\x) \|_{p,q}$, $b = \| \nabla g_t(\xo) \|_p$, and $\gamma = g_t(\xo)$.
\end{theorem}

\begin{proof}
By holder's inequality and the definition of induced norm, we have
\begin{equation*}
    | \nabla g_t(\xo)^\top \del | \leq \| \nabla g_t(\xo) \|_q \| \del \|_p 
\end{equation*}
and 
\begin{align*}
    | \del^\top \bm{H}(\xo+s\del) \del | &\leq \| \bm{H}(\xo+s\del) \del \|_q \| \del \|_p \\
    &\leq \| \bm{H}(\xo+s\del)\|_{p,q} \|\del \|_p \|\del \|_p \\
    &\leq \max_{\x \in \Ball} \|\bm{H}(\x)\|_{p,q} \|\del \|_p^2.
\end{align*}
Let $a = \max_{\x \in \Ball} \| \bm{H}(\x) \|_{p,q}$, $b = \| \nabla g_t(\xo) \|_p$, and $\gamma = g_t(\xo)$, we get a lower bound of $g_t(\xo+\del)$: 
\begin{align}
    g_t(\xo+\del) &= g_t(\xo) + \nabla g_t(\xo)^\top \del + \frac{1}{2} \del^\top \bm{H}(\xo+s\del) \del \nonumber \\
    & \geq  g_t(\xo) - b \| \del \|_p -\frac{1}{2} a \| \del \|_p^2. \label{eq:2nd_order_ineq} 
\end{align}
If we can guarantee \eqref{eq:2nd_order_ineq} $\geq 0$, then we can guarantee $g_t(\xo+\del) \geq 0$, which is the definition of targetted robustness in \eqref{eq:target_robustness}. Thus, the condition of \eqref{eq:2nd_order_ineq} $\geq 0$ gives 
\begin{equation*}
   \| \del \|_p \leq \frac{-b+\sqrt{b^2+2a\gamma}}{a}. 
\end{equation*}
\end{proof}

\subsection{Sampling via Extreme Value Theory}
\label{sec:sample_evt}

Theorem~\ref{thm:delta_bnd} needs the value $a \coloneqq \max_{\x \in \Ball} \|\bm{H}(\x)\|_{p,q}$, which is the maximum subordinate norm of the Hessian matrix within $\x \in \Ball$. When $p=q=2$, it becomes the well-known spectral norm, and can be evaluated efficiently on a single point $\x$ using power iteration or Lanczos method. Under the framework of CLEVER, we apply extreme value theory to estimate $a$ by sampling different $\x \in \Ball$ and running power iterations on each sampled point. In this paper, we focus on the case of $p=q=2$ only ($\ell_2$ robustness). After we get an estimate of $a$, a second order robustness lower bound can be estimated at point $\xo$ using~\eqref{eq:2nd_order_bnd}. The estimated bound of \eqref{eq:1st_order_bnd} is named \textit{1st-order} CLEVER while the estimated bound of \eqref{eq:2nd_order_bnd} is called \textit{2nd-order} CLEVER. 


\section{CLEVER with gradient masking based defense}
\label{sec:typestyle}

\subsection{Gradient Masking}

Gradient masking~\cite{tramer2017ensemble} is a popular defending method against adversarial examples where the model does
not provide useful gradients for generating adversarial examples. Typical gradient masking techniques include adding non-differentiable layers~\cite{guo2017countering} (bit-depth reduction, JPEG compression, etc) to the network, numerically making the gradient vanish (Defensive Distillation~\cite{papernot2016distillation}), and modifying the optimization landscape of the loss function in a local region~\cite{tramer2017ensemble} of each data point. These methods typically prevent gradient-based adversarial attacks by providing non-informative gradients. However, many of the gradient masking techniques have been shown ineffective as a defense. Notably, Defensive Distillation can be bypassed by attacking the logit (unnormalized probability) layer values to avoid the saturated softmax functions; many non-differentiable transformation functions can be bypassed using the Backward Pass Differentiable Approximation (BPDA)~\cite{athalye2018obfuscated}; the modifications in local landscape of the loss function can be escaped by adding a small random noise when performing the attack~\cite{tramer2017ensemble}.

When CLEVER is evaluated, we always use the logit layer values, thus we are not subject to the saturation of the sigmoid units. Additionally, during the sampling processes, we evaluate gradients using a large number of randomly perturbed images, thus CLEVER is likely to escape the region of masked gradients in local loss landscape. The remaining concern is thus whether CLEVER can be evaluated on networks with a non-differentiable layer as a defense. For example, if the input image is quantized via bit-depth reduction, a staircase function is applied to the network and thus its gradient cannot be computed via automatic differentiation. We will formally discuss this situation in the next section.

\subsection{Apply Backward Pass Differentiable Approximation (BPDA) to CLEVER}

For a neural network classifier $f(\bm{x})$, we can apply a non-differentiable transformation $h(\bm{x})$ to the input $\bm{x}$ and then feed the data after transformation into $f$. The function $f(h(\bm{x}))$ thus becomes non-differentiable, and gradient based adversarial attacks fail to find successful adversarial examples. An example of $h(\bm{x})$ is a staircase function, as suggested in~\cite{goodfellow2018gradient}. This transformation also hinders the direct use of CLEVER to evaluate the robustness of $f(h(\bm{x}))$. 

To handle non-differentiable transformations, we use the Backward Pass Differentiable Approximation (BPDA)~\cite{athalye2018obfuscated} technique. The intuition behind BPDA is that although $h(\bm{x_0})$ is non-differentiable (e.g., bit-depth reduction, JPEG compression, etc), it usually holds that $h(\bm{x_0}) \approx \bm{x_0}$. Thus, in backpropagation, we can assume that

\begin{equation}
    \left. \nabla_x f(h(\bm{x})) \right|_{\bm{x}=\xo} \approx \left. \nabla_x f(\bm{x}) \right|_{\bm{x}=h(\xo)}.
\end{equation}

To evaluate CLEVER for a network with an input transformation $h$ (for example, a staircase function), $\bm{x}$ is sampled within an $\ell_p$ ball around $\xo$. Then, a transformation $h(\bm{x_0})$  is applied, such that $\hat{\bm{x}}=h(\bm{x})$. Then, the backpropagation procedure computes $\nabla_{\hat{\bm{x}}} f(\hat{\bm{x}})$. We simply collect $\nabla_{\hat{\bm{x}}} f(\hat{\bm{x}})$ as the gradient, and compute its norm as a sample for Lipschitz constant estimation.

\subsection{CLEVER is a White-Box Evaluation Tool}

CLEVER is intended to be a tool for network designers and to evaluate network robustness in the ``white-box'' setting in which we know how a (defended) neural network processes the input. In this case, we can deal with the non-differentiable transformation $h$ with BPDA, and evaluate the \textit{intrinsic} robustness of the model, without the ``False Sense of Security~\cite{athalye2018obfuscated}'' provided by gradient masking.

In black-box attack setting, the gradient of $f(h(\bm{x}))$ must be evaluated via finite differences~\cite{CPY17zoo}, thus a non-differentiable $g(\bm{x})$ prevents gradient based attacks in black-box settings because the estimated gradient becomes infinite (i.e., the value of $f(g(\bm{x}))$ is unlikely to change when $\bm{x}$ is changed by a small amount). Goodfellow \cite{goodfellow2018gradient} raises concerns on the effectiveness of CLEVER in this setting, but this setting is different from our intended usage of CLEVER. Most importantly, CLEVER computes gradients using backpropagation via automatic differentiation in the white-box setting, rather than using finite differences. Despite the limited numerical precision on digital computers, CLEVER is not subject to the same numerical issues as in the black-box attack setting. Unless backpropagation fails, CLEVER is able to estimate a reasonable robustness score reflecting the intrinsic model robustness.

\section{Experiments}
\label{sec:exp}

\subsection{Experiments on 1st Order and 2nd Order Bounds}
We compute the targeted robustness bounds for a 7-layer CNN model with tanh activations (which is twice differentiable) on CIFAR dataset with a validation accuracy of 72.6\%. 
We calculated both Eq. \eqref{eq:1st_order_bnd} and \eqref{eq:2nd_order_bnd} via sampling with extreme value theory, and we denote the estimated scores as ``1st order'' and ``2nd order'' CLEVER scores respectively in the Tables. In particular, we follow the sampling procedure proposed in \cite{weng2018evaluating} to estimate the Lipschitz constant by fitting the samples with maximum likelihood estimation on Reversed Weibull distribution and calculate the estimated robustness scores of \eqref{eq:1st_order_bnd}. For the ``2nd order'' bound \eqref{eq:2nd_order_bnd}, we also use sampling and extreme value theory to calculate the estimated bounds, as describe in Section~\ref{sec:sample_evt}. 
For fair comparison, we use the same number of samples ($N_b = 100$ and $N_s = 200$) for both estimated bounds and we compare their average as well as the percentage of image examples that the score is larger than the other. For each image, we select three attack target classes: least likely, random and runner-up. The results are summarized in Tables \ref{tab:cifar-least}, \ref{tab:cifar-top2} and \ref{tab:cifar-random}. We observe that the 1st order and 2nd order average CLEVER scores usually stay close, indicating that both scores agree with each other. 

Since CLEVER is a score of estimated lower bound, we desire the score is not trivially small, but smaller than the upper bound found by adversarial attacks (in our case the CW $\ell_2$ attack). As shown in Tables \ref{tab:cifar-least}, \ref{tab:cifar-top2} and \ref{tab:cifar-random}, all CLEVER scores are less then CW $\ell_2$ distortion. Second order CLEVER can sometimes give a better result than its first order counterpart, indicating that second order approximation is probably more accurate for these examples. The ``avg. \% of increase on the score'' rows in tables report the improvement of score when one method is better than the other; for example, in runner-up target, second order CLEVER increases the score for 82\% of the examples, and the average improvement of score comparing to first order CLEVER is 58\%.

\setlength{\floatsep}{10pt}
\setlength{\textfloatsep}{10pt}

\begin{table}[t]
\centering
\caption{Comparison of 1st order and 2nd order $\ell_2$ CLEVER with least-likely target labels on a 7-layer $\tanh$ CIFAR CNN. The average distortion found by CW-$\ell_2$ attack is 0.310.}
\begin{tabular}{l|c|c}
\hline
Least-likely Target           & 1st order & 2nd order \\ \hline
avg $\ell_2$ CLEVER               & 0.057     & 0.051     \\ \hline
\% of images with larger score              & 54        & 46        \\ \hline
avg \% of increase on the score  & 47\%    & 44\%    \\ \hline
\end{tabular}
\label{tab:cifar-least}
\end{table}

\begin{table}[t]
\centering
\caption{Comparison of 1st order and 2nd order $\ell_2$ CLEVER with runner-up target labels on a 7-layer $\tanh$ CIFAR CNN. The average distortion found by CW-$\ell_2$ attack is 0.101.}
\begin{tabular}{l|c|c}
\hline
Runner-up Target           & 1st order & 2nd order \\ \hline
avg $\ell_2$ CLEVER                 & 0.024     & 0.026     \\ \hline
\% of images with larger score              & 18        & 82        \\ \hline
avg \% of increase on the score & 77\%    & 58\%    \\ \hline
\end{tabular}
\label{tab:cifar-top2}
\end{table}

\begin{table}[t]
\centering
\caption{Comparison of 1st order and 2nd order $\ell_2$ CLEVER with random target labels on a 7-layer $\tanh$ CIFAR CNN. The average distortion found by CW-$\ell_2$ attack is 0.264.}
\begin{tabular}{l|c|c}
\hline
Random Target          & 1st order  & 2nd order \\ \hline
avg $\ell_2$ CLEVER                & 0.049     & 0.036     \\ \hline
\% of images with larger score              & 76        & 24        \\ \hline
avg \% of increase on the score & 55\%    & 68\%    \\ \hline
\end{tabular}
\label{tab:cifar-random}
\end{table}

\subsection{Experiments on Networks with Input Transformation as a Gradient Masking based Defense}

We conduct experiments on a 121-layer DenseNet~\cite{huang2017densely} network pretrained on ImageNet dataset\footnote{model available at \url{https://github.com/pudae/tensorflow-densenet}}. We employ two non-differentiable input transfomrations that mask gradients: bit-depth reduction (reducing each color channel from 8-bit to 3-bit, setting all lower bits to 0) and JPEG compression (quality set to 75\%).
We compute $\ell_2$ CLEVER (first order) scores for the network with and without input transformations, with CLEVER parameter $N_b = 200$ and $N_s = 1024$. We randomly choose 100 images from the ImageNet validation set, and select three attack target classes for each image (least likely, random and runner-up). Misclassified images are skipped.


Table~\ref{tb:input_transform} compares the $\ell_2$ CLEVER scores for three target classes, for the original model, and for bit-depth reduction or JPEG compression as input transformations. BPDA is used to compute CLEVER when an input transformation is applied. Not surprisingly, the CLEVER scores for networks with input transformation as a gradient masking method do not noticeably increase, indicating that these transformations do not increase the model's intrinsic robustness; in other words, with BPDA applied, we can still obtain similar gradients as the original model, thus it is expected that CLEVER scores do not change too much in this situation.

\begin{table}[t]
\label{tb:input_transform}
\caption{$\ell_2$ robustness CLEVER scores with and without input transformations on a 121-layer Densenet model, for three different target classes. The average adversarial distortion of CW $\ell_2$ attack for the same set of images are 0.2058, 0.52788 and 0.66114, for runner-up, random and least-likely target classes, respectively.}
\centering
\begin{adjustbox}{max width=\linewidth}
\begin{tabular}{|l|c|c|c|}
\hline
Target Class        & Runner-up & Random  & Least Lilely \\ \hline
No transformation   & 0.14229   & 0.35632 & 0.44725      \\ \hline
Bit-depth reduction & 0.10223   & 0.26224 & 0.34722      \\ \hline
JPEG compression    & 0.11539   & 0.27804 & 0.36275      \\ \hline
\end{tabular}
\end{adjustbox}
\end{table}

\section{Conclusions}

CLEVER~\cite{weng2018evaluating} is a first-order approximation based robustness score. We move one step further to give a second order formal guarantee for DNN robustness. We show that it improves the estimated robustness lower bound for some examples, and in many cases both first and second order CLEVER scores are coherent.
Additionally, we successfully apply Backward Pass Differentiable Approximation (BPDA) to compute CLEVER scores for a network with non-differentiable input transformations, including staircase functions. Our discussions and results remedy the concerns raised in~\cite{goodfellow2018gradient}.

\section{Acknowledgement}
Tsui-Wei Weng and Luca Daniel acknowledge partial support of MIT IBM Watson AI Lab. 


\bibliographystyle{IEEEbib}
\bibliography{refs}

\begin{thebibliography}{10}

\bibitem{weng2018evaluating}
Tsui-Wei Weng, Huan Zhang, Pin-Yu Chen, Jinfeng Yi, Dong Su, Yupeng Gao,
  Cho-Jui Hsieh, and Luca Daniel,
\newblock ``Evaluating the robustness of neural networks: An extreme value
  theory approach,''
\newblock {\em Sixth International Conference on Learning Representations
  (ICLR)}, 2018.

\bibitem{carlini2017towards}
Nicholas Carlini and David Wagner,
\newblock ``Towards evaluating the robustness of neural networks,''
\newblock in {\em IEEE Symposium on Security and Privacy (SP)}, 2017, pp.
  39--57.

\bibitem{bastani2016measuring}
Osbert Bastani, Yani Ioannou, Leonidas Lampropoulos, Dimitrios Vytiniotis,
  Aditya Nori, and Antonio Criminisi,
\newblock ``Measuring neural net robustness with constraints,''
\newblock in {\em Advances in Neural Information Processing Systems}, 2016, pp.
  2613--2621.

\bibitem{chen2017ead}
Pin-Yu Chen, Yash Sharma, Huan Zhang, Jinfeng Yi, and Cho-Jui Hsieh,
\newblock ``Ead: elastic-net attacks to deep neural networks via adversarial
  examples,''
\newblock {\em arXiv preprint arXiv:1709.04114}, 2017.

\bibitem{moosavi2016deepfool}
Seyed-Mohsen Moosavi-Dezfooli, Alhussein Fawzi, and Pascal Frossard,
\newblock ``Deepfool: a simple and accurate method to fool deep neural
  networks,''
\newblock in {\em IEEE Conference on Computer Vision and Pattern Recognition},
  2016, pp. 2574--2582.

\bibitem{athalye2018obfuscated}
Anish Athalye, Nicholas Carlini, and David Wagner,
\newblock ``Obfuscated gradients give a false sense of security: Circumventing
  defenses to adversarial examples,''
\newblock {\em 35th International Conference on Machine Learning (ICML)}, 2018.

\bibitem{athalye2018robustness}
Anish Athalye and Nicholas Carlini,
\newblock ``On the robustness of the cvpr 2018 white-box adversarial example
  defenses,''
\newblock {\em arXiv preprint arXiv:1804.03286}, 2018.

\bibitem{carlini2017magnet}
Nicholas Carlini and David Wagner,
\newblock ``Magnet and" efficient defenses against adversarial attacks" are not
  robust to adversarial examples,''
\newblock {\em arXiv preprint arXiv:1711.08478}, 2017.

\bibitem{carlini2017adversarial}
Nicholas Carlini and David Wagner,
\newblock ``Adversarial examples are not easily detected: Bypassing ten
  detection methods,''
\newblock {\em arXiv preprint arXiv:1705.07263}, 2017.

\bibitem{katz2017reluplex}
Guy Katz, Clark Barrett, David~L Dill, Kyle Julian, and Mykel~J Kochenderfer,
\newblock ``Reluplex: An efficient smt solver for verifying deep neural
  networks,''
\newblock in {\em International Conference on Computer Aided Verification}.
  Springer, 2017, pp. 97--117.

\bibitem{weng2018towards}
Tsui-Wei Weng, Huan Zhang, Hongge Chen, Zhao Song, Cho-Jui Hsieh, Duane Boning,
  Inderjit~S Dhillon, and Luca Daniel,
\newblock ``Towards fast computation of certified robustness for relu
  networks,''
\newblock {\em 35th International Conference on Machine Learning (ICML)}, 2018.

\bibitem{hein2017formal}
Matthias Hein and Maksym Andriushchenko,
\newblock ``Formal guarantees on the robustness of a classifier against
  adversarial manipulation,''
\newblock in {\em Advances in Neural Information Processing Systems}, 2017, pp.
  2263--2273.

\bibitem{goodfellow2018gradient}
Ian Goodfellow,
\newblock ``Gradient masking causes clever to overestimate adversarial
  perturbation size,''
\newblock {\em arXiv preprint arXiv:1804.07870}, 2018.

\bibitem{tramer2017ensemble}
Florian Tram{\`e}r, Alexey Kurakin, Nicolas Papernot, Dan Boneh, and Patrick
  McDaniel,
\newblock ``Ensemble adversarial training: Attacks and defenses,''
\newblock {\em Sixth International Conference on Learning Representations
  (ICLR)}, 2018.

\bibitem{guo2017countering}
Chuan Guo, Mayank Rana, Moustapha Cisse, and Laurens van~der Maaten,
\newblock ``Countering adversarial images using input transformations,''
\newblock {\em arXiv preprint arXiv:1711.00117}, 2017.

\bibitem{papernot2016distillation}
Nicolas Papernot, Patrick McDaniel, Xi~Wu, Somesh Jha, and Ananthram Swami,
\newblock ``Distillation as a defense to adversarial perturbations against deep
  neural networks,''
\newblock in {\em IEEE Symposium on Security and Privacy (SP)}, 2016, pp.
  582--597.

\bibitem{CPY17zoo}
Pin-Yu Chen, Huan Zhang, Yash Sharma, Jinfeng Yi, and Cho-Jui Hsieh,
\newblock ``{ZOO}: Zeroth order optimization based black-box attacks to deep
  neural networks without training substitute models,''
\newblock in {\em ACM Workshop on Artificial Intelligence and Security}, 2017,
  pp. 15--26.

\bibitem{huang2017densely}
Gao Huang, Zhuang Liu, Kilian~Q Weinberger, and Laurens van~der Maaten,
\newblock ``Densely connected convolutional networks,''
\newblock in {\em Proceedings of the IEEE conference on computer vision and
  pattern recognition}, 2017, vol.~1, p.~3.

\end{thebibliography}

\end{document}